\documentclass[conference]{IEEEtran}
\usepackage{cite}

\ifCLASSINFOpdf
   \usepackage[pdftex]{graphicx}
\else
\fi
\usepackage[cmex10]{amsmath}
\usepackage{amsfonts}
\usepackage{amssymb}
\usepackage{bm}
\usepackage{amsthm}

\usepackage{algorithmicx}
\usepackage{algpseudocode}
\algdef{SE}[DOWHILE]{Do}{doWhile}{\algorithmicdo}[1]{\algorithmicwhile\ #1}%
\usepackage[caption=false,font=footnotesize]{subfig}
\usepackage{url}

\usepackage{enumerate}
\usepackage{mathtools}  
\usepackage{graphicx}
\usepackage{tikz}
\usetikzlibrary{shapes,arrows,automata}
\usepackage{thmtools}  

\newtheorem{theorem}{Theorem}{\bfseries}{\rmfamily}
\newtheorem{definition}{Definition}{\bfseries}{\rmfamily}
\newtheorem{axiom}{Axiom}{\bfseries}{\rmfamily}
\newtheorem{example}{Example}{\itshape}{\rmfamily}
\renewcommand\thmcontinues[1]{Continued}
{\itshape\bfseries}{\rmfamily}
{\bfseries}{\rmfamily}
{\bfseries}{\rmfamily}
{\bfseries}{\rmfamily}
{\bfseries}{\rmfamily}
\newtheorem{proposition}[theorem]{Proposition}{\bfseries}{\rmfamily}
{\bfseries}{\rmfamily}
{\itshape}{\rmfamily}
{\bfseries}{\rmfamily}
{\bfseries}{\rmfamily}
{\itshape}{\rmfamily}
\renewcommand{\IEEEQED}{\IEEEQEDopen}

{\bfseries}{\itshape}

\def\sgn{\mathop{\rm sgn}}

\begin{document}
%
\title{Some Supplementaries to The Counting Semantics\\ for Abstract Argumentation}

\author{\IEEEauthorblockN{Fuan Pu, Jian Luo and Guiming Luo}
\IEEEauthorblockA{School of Software, Tsinghua University\\
Beijing China, 100084\\
pfa12@mails.tsinghua.edu.cn, j-luo10@mails.tsinghua.edu.cn, gluo@tsinghua.edu.cn}
}


%


\maketitle

\begin{abstract}
Dung's abstract argumentation framework consists of a set of interacting arguments
and a series of semantics for evaluating them. Those semantics partition the powerset of the set of arguments into two classes: extensions and non-extensions. In order to reason with a specific semantics, one needs to take a credulous or skeptical approach, i.e. an argument is eventually accepted, if it is accepted in one or all extensions, respectively. In our previous work \cite{ref-pu2015counting}, we have proposed a novel semantics, called \emph{counting semantics}, which allows for a more fine-grained assessment to arguments by counting the number of their respective attackers and defenders based on argument graph and argument game. In this paper, we continue our previous work by presenting some supplementaries about how to choose the damaging factor for the counting semantics, and what relationships with some existing approaches, such as Dung's classical semantics, generic gradual valuations. Lastly, an axiomatic perspective on the ranking semantics induced by our counting semantics are presented.
\end{abstract}


%
\IEEEpeerreviewmaketitle

\section{Introduction}
Argumentation is closer to human reasoning than classical logic. It provides the means for comparing information by analysing pros and cons when trying to make a decision \cite{ref-David00argknowledge}. Argumentation theory has gained interest in artificial intelligence since it provides the basis for computational models inspired by the way humans reason. These models have extended classical reasoning approaches, based on deductive logic, that were proving increasingly inadequate for problems requiring non-monotonic reasoning and explanatory reasoning not available in standard nonmonotonic logics.

The most popularly used framework to talk about general issues of argumentation is that of Dung's abstract argumentation \cite{ref-Dung1995AAF}, in which arguments are represented as atomic entities and the interactions between different arguments are modeled by an attack relation. It provides a series of extension-based semantics for solving the inconsistent knowledges by selecting acceptable subsets. Generally, for a specific extension-based semantics, there is usually a set of extensions that is consistent with the semantical context. In order to reason with a semantics one has to take either a credulous or skeptical perspective. That is, an argument is ultimately accepted with respect to a semantics if the argument is contained in at least one extension consistent with that semantics (the credulous perspective) or if the argument is included in all extensions consistent with the semantics (the skeptical perspective). This extreme points of views may cause undesired results, since in extreme cases the set of credulously accepted arguments may contain nearly the whole set of arguments and the set of skeptically accepted set of arguments may be nearly empty.

In order to get a more fine-grained assessment on arguments, we had proposed a new semantics that generalized the classical extension-based semantics in \cite{ref-pu2015counting}. Our proposal is based on counting the number of attackers and defenders for each argument, hence called \emph{counting semantics}. An argument is more acceptable if there are more defenders for it and less attackers against it. In this paper, we will continue our previous work to discuss the determination of the damping factor for the counting semantics and relate our proposal with some existing approaches.

The rest of this paper is structured as follows. In Section \ref{Sec_Background}, we give a brief overview on abstract argumentation. Section~\ref{Sec_CountingSemantics} provides the basic concept of the counting semantics. Section~\ref{Sec_DampingFactor} discusses the selection of the damping factor. We relate our proposal with classical (extension-based) semantics in Section~\ref{Sec_Comparation_Classical} and with the generic gradual valuation in Section~\ref{Sec_Comparation_Gradual}. In Section~\ref{Sec_AxiomOnArgRank}, we present an axiomatic perspective on the ranking semantics induced by our counting semantics. We conclude in Section~\ref{Sec_Conclusion}.

\section{Abstraction Argumentation Framework} \label{Sec_Background}
In this section, we briefly outline the key elements of abstract argumentation frameworks. Now let us begin Dung's abstract characterization of an argumentation system \cite{ref-Dung1995AAF}:
\begin{definition}[Abstract Argumentation Framework] \label{Def_Argumentation_Framework}
An \emph{argumentation framework} is a pair $\textit{AF}=\left< \mathcal{X}, \mathcal{R}\right>$ where $\mathcal{X}$ is a finite set of arguments and $\mathcal{R} \subseteq \mathcal{X} \times \mathcal{X}$ is a binary relation on $\mathcal{X}$, also called \emph{attack relation}. $(a,b)\in \mathcal{R}$ means that $a$ attacks $b$, or $a$ is an attacker of $b$. Often, we write $(a,b) \in \mathcal{R}$ as $a \mathcal{R} b$.
\end{definition}

We denote by $\mathcal{R}^-(x)$ (respectively, $\mathcal{R}^+(x)$) the subset of $\mathcal{X}$ containing those arguments that attack (respectively, are attacked by) the argument $x\in\mathcal{X}$, extending this notation in the natural way to sets of arguments, so that for $S\subseteq \mathcal{X}$, $\mathcal{R}^-(S) \triangleq \{x\in\mathcal{X}: \exists y \in S \mbox{~such that~} x\mathcal{R}y\}$ and $\mathcal{R}^+(S) \triangleq \{x\in\mathcal{X}: \exists y \in S \mbox{~such that~} y\mathcal{R}x\}$. Now, let us characterise two fundamental notions of conflict-free and defence.
\begin{definition}[Conflict-free, Defense] \label{Def_CF&Defense}
Let $\textit{AF}=\left< \mathcal{X}, \mathcal{R}\right>$ be an argumentation framework, let $S\subseteq\mathcal{X}$ and $x\in\mathcal{X}$.
\begin{itemize}
  \item $S$ is \emph{conflict-free} iff $S \cap \mathcal{R}^-(S)=\emptyset$.
  \item $S$ \emph{defends} argument $x$ iff $\mathcal{R}^-(x)\subseteq\mathcal{R}^+(S)$. It is also said that argument $x$ is \emph{acceptable} with respect to $S$.
\end{itemize}
\end{definition}
Obviously, a set of arguments is conflict-free iff no argument in that set attacks another. A set of arguments defends a given argument iff it attacks all its attackers.

\begin{definition}[Characteristic Function] \label{Def_CharacteristicFun}
The \emph{characteristic function} of an argumentation framework $\left< \mathcal{X}, \mathcal{R}\right>$ is a function $\mathfrak{F}: 2^{\mathcal{X}}\mapsto 2^{\mathcal{X}}$ such that, given $S\subseteq\mathcal{X}$, $\mathfrak{F}(S)=\{x\in\mathcal{X}| S \hbox{~defends~} x\}$.
\end{definition}
Stated otherwise, $\mathfrak{F}(S)$ is the set of all arguments that $S$ defends. To define the solutions of an argumentation framework, we mean choosing a set of arguments that satisfies some acceptable criteria. Several of these properties, called extensions or semantics, have been proposed by Dung.
\begin{definition}[Acceptability Semantics] \label{Def_Acceptability}
Let $S\subseteq\mathcal{X}$ be a conflict-free set of arguments in argument system $\left< \mathcal{X}, \mathcal{R}\right>$.
\begin{itemize}
  \item $S$ is an \emph{admissible extension} iff $S\subseteq\mathfrak{F}(S)$.
  \item $S$ is a \emph{complete extension} iff $S=\mathfrak{F}(S)$.
  \item $S$ is a \emph{grounded extension} iff $S=\mathfrak{F}(S)$ and $S$ is minimal (w.r.t. $\subseteq$). It is the least fixed point of $\mathfrak{F}$, and its existence and uniqueness have been proved in \cite{ref-Dung1995AAF} and \cite{ref-caminada2006issue}.
  \item $S$ is a \emph{preferred extension} iff $S=\mathfrak{F}(S)$ and $S$ is maximal (w.r.t. $\subseteq$).
  \item $S$ is a \emph{stable extension} iff $S=\mathcal{X}\backslash \mathcal{R}^+(S)$ or $S=\overline{\mathcal{R}^+(S)}$, where the bar on $\mathcal{R}^+(S)$ denotes the relative complement of $\mathcal{R}^+(S)$ with respect to $\mathcal{X}$.
\end{itemize}
\end{definition}

\begin{example} \label{Exp_SimpleAF}
  Consider the argumentation framework $\textit{AF}=\left< \mathcal{X}, \mathcal{R}\right>$, depicted in Fig.~\ref{Fig_AAF}, in which $\mathcal{X}=\{x_1, x_2, x_3, x_4\}$ and $\mathcal{R}=\{(x_2,x_1),(x_3,x_2),(x_2,x_3),(x_3,x_3), (x_4,x_2)\}$. For this example, $\textit{AF}$ has two admissible sets: $\{x_4\}$ and $\{x_1, x_4\}$. $\{x_1,x_4\}$ is the only preferred extension of $\textit{AF}$, and it is also complete and grounded. $\textit{AF}$ has no stable extension.
\end{example}

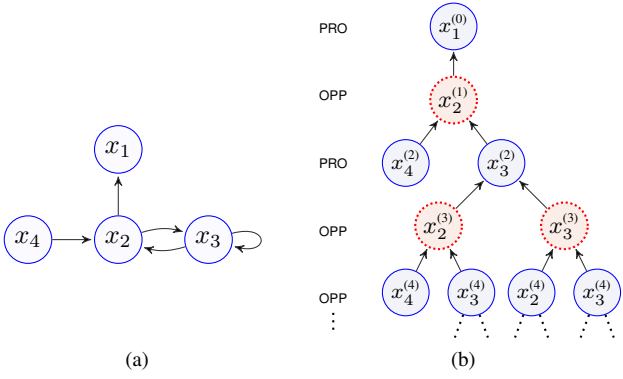
\begin{figure}[!tb]
\centering
\subfloat[]{
\begin{tikzpicture}[->,>=stealth',shorten >=1pt,auto,node distance=1.2cm, thin]
\tikzstyle{vNorm}=[draw=blue,fill=blue!2,circle,text width=5mm,inner sep=1pt,minimum height=6pt, align=center]
\tikzstyle{vBlank}=[text width=5mm, font={\tiny\sf}, align=center] 
\tikzstyle{every edge}=[draw=black!80, font=\scriptsize]

\node[vNorm](x1){$x_1$};
\node[vNorm, below of=x1](x2){$x_2$};
\node[vNorm, right of=x2](x3){$x_3$};
\node[vNorm, left of=x2](x4){$x_4$};

\path (x2) edge               (x1)
      (x2) edge [bend left=17](x3)
      (x3) edge [bend left=17](x2)
      (x3) edge [loop right]  (x3)
      (x4) edge               (x2);

\node[vBlank, below of=x2](B1){};

\end{tikzpicture}
\label{Fig_AAF}}
\hfil
\subfloat[]{
\begin{tikzpicture}[auto,node distance=0.9cm]
\definecolor{Attacker}{rgb}{0.85,0.33,0.1}
\definecolor{Defender}{rgb}{0.08,0.17,0.55}
\tikzstyle{vBlue}=[draw=blue,fill=Defender!5, thin, circle,inner sep=0.2pt, minimum height=6pt, text width=5mm, font=\footnotesize, align=center]
\tikzstyle{vRed}=[draw=red,densely dotted,thick, fill=Attacker!10,circle,inner sep=0.2pt, minimum height=6pt, text width=5mm, font=\small, align=center]
\tikzstyle{vTxt}=[text width=5mm, font={\tiny\sf}, align=center]
\tikzstyle{every edge}=[draw=black!80, ->,>=stealth',shorten >=0.5pt, font=\footnotesize]

\node[vTxt](P1){PRO};
\node[vTxt, below of=P1](O1){OPP};
\node[vTxt, below of=O1](P2){PRO};
\node[vTxt, below of=P2](O2){OPP};
\node[vTxt, below of=O2](P3){OPP};
\draw[dotted, color=black, thick] (node cs:name=P3) -- +(0,-4.5mm);

\node[vBlue, right of=P1, xshift=7mm, yshift=0.2mm](x11){$x_{\scriptstyle 1}^{\mbox{\tiny (0)}}$};
\node[vRed, below of=x11, yshift=-0.8mm](x21){$x_{\scriptstyle 2}^{\mbox{\tiny (1)}}$};
\node[vBlue, below left of=x21, yshift=-2mm](x31){$x_{\scriptstyle 4}^{\mbox{\tiny (2)}}$};
\node[vBlue, below right of=x21, yshift=-2mm](x32){$x_{\scriptstyle 3}^{\mbox{\tiny (2)}}$};
\node[vRed, below left of=x32, yshift=-2mm, xshift=-2mm](x41){$x_{\scriptstyle 2}^{\mbox{\tiny (3)}}$};
\node[vRed, below right of=x32, yshift=-2mm, xshift=2mm](x42){$x_{\scriptstyle 3}^{\mbox{\tiny (3)}}$};
\node[vBlue, below left of=x41, yshift=-2.5mm, xshift=2mm](x51){$x_{\scriptstyle 4}^{\mbox{\tiny (4)}}$};
\node[vBlue, below right of=x41, yshift=-2.5mm, xshift=-2mm](x52){$x_{\scriptstyle 3}^{\mbox{\tiny (4)}}$};
\node[vBlue, below left of=x42, yshift=-2.5mm, xshift=2mm](x53){$x_{\scriptstyle 2}^{\mbox{\tiny (4)}}$};
\node[vBlue, below right of=x42, yshift=-2.5mm, xshift=-2mm](x54){$x_{\scriptstyle 3}^{\mbox{\tiny (4)}}$};
\path (x21) edge               (x11)
      (x31) edge               (x21)
      (x32) edge               (x21)
      (x41) edge               (x32)
      (x42) edge               (x32)
      (x51) edge               (x41)
      (x52) edge               (x41)
      (x53) edge               (x42)
      (x54) edge               (x42);

\draw[dotted, color=black, thick] (node cs:name=x52) -- +(-2.2mm,-6mm);
\draw[dotted, color=black, thick] (node cs:name=x52) -- +(2.2mm,-6mm);
\draw[dotted, color=black, thick] (node cs:name=x53) -- +(-2.2mm,-6mm);
\draw[dotted, color=black, thick] (node cs:name=x53) -- +(2.2mm,-6mm);
\draw[dotted, color=black, thick] (node cs:name=x54) -- +(-2.2mm,-6mm);
\draw[dotted, color=black, thick] (node cs:name=x54) -- +(2.2mm,-6mm);

\end{tikzpicture}
\label{Fig_DisputeTree}}
\caption{Argumentation framework and dispute tree. (a) shows an argumentation framework, (b) presents the dispute tree induced in $x_1$.}
\label{Fig_AF&DisputeTree}
\end{figure}

\section{Counting Semantics for Argumentation}  \label{Sec_CountingSemantics}
In classical abstract argumentation, arguments are either acceptable or unacceptable, given a specific semantics. In order to get a more fine-grained view on the status of arguments we had proposed a new semantics that generalized classical semantics \cite{ref-pu2015counting}, called counting semantics. The counting semantics assigns to each argument of an argumentation framework a numerical strength value which is meant to be interpreted as a degree of acceptability so as to finely compare and rank arguments from the most acceptable to the weakest one(s). In this section, we provide the basic concepts of the counting semantics and supplement some examples to present the calculation of the counting semantics.

The fundamental intuition used to formalise the degree of acceptability is essentially the same as those found in abstract argumentation theory: argument $x$ is more acceptable than argument $y$ iff $x$ has a better defence (for it) and a lower attack (against it). In order to assess the strength value of each argument in an argumentation framework, we consider their evaluation procedures as dialogue games between two fictitious agents --- \textsf{PRO} (for ``proponent'') and \textsf{OPP} (for ``opponent'') --- each of which are referred to as the other's ``counterpart'' \cite{Simari2009argame,ref-caminada2009argame}. A dialogue game begins with \textsf{PRO} putting forward an initial argument, and then \textsf{PRO} and \textsf{OPP} take turns in a sequence of moves called a \emph{dispute}, in which each agent makes an argument that attacks its counterpart's last move. In general, the counterpart can try a different line of attack and create a new dispute. This leads to a \emph{dispute tree} structure that represents the dialogue game. Nodes in a dispute tree are labelled by arguments and are assigned the status of \emph{defender} and \emph{attacker} of the root argument, depending upon the argument at that node is made by the proponent or the opponent, or depending upon whether the walk length from the current node to the root node is even or odd. For instance, consider two agents arguing the argumentation framework shown in Fig.~\ref{Fig_AAF}, and the dispute tree induced by argument $x_1$ is shown Fig.~\ref{Fig_DisputeTree}. Clearly, this dispute tree is infinite as both agents are able to repeat counterarguments due to the presence of cycles in the argument graph. In this dispute tree, the blue solid nodes, made by \textsf{PRO}, are defenders of $x_1$, whereas, the red dotted nodes, made by \textsf{OPP}, are attackers of $x_1$. Each node is also assigned a superscript, which denotes the length of the walk from the current node to the root node. We can see that if a node has a even-numbered superscript then it is a defender, otherwise it is an attacker. Note that the root node is also a defender of itself as each argument has a walk with length $0$ to itself.

We claim that an argument is more acceptable if \textsf{PRO} puts forward more number of defenders for it and \textsf{OPP} puts forward less number of attackers against it. One simple approach is thus to count all length of attackers and defenders for each argument. We positively count all defenders and negatively count all attackers. This is easy to understand since an argument is always weakened by its attackers and is ``reinstated'' by its defenders. Therefore, in any case, the greater the number computed, the more acceptable the argument. Let $\textit{AF}=\left< \mathcal{X}, \mathcal{R}\right>$ with $\mathcal{X}=\{x_1,x_2,\cdots,x_n\}$. We define the \emph{attack matrix} $\bm{A}$ for $\textit{AF}$ as a $n\times n$ matrix such that $a_{ij}=1$ if $x_j\mathcal{R}x_i$; otherwise, $0$.\footnote{In fact, the attack matrix of an argumentation framework is the transpose of the adjacency matrix of its corresponding attack graph.} The number of $\ell$-length attackers (when $\ell$ is odd) or defenders (when $\ell$ is even) for each argument in $\mathcal{X}$ are stored by employing a $n$-dimensional column vector $\bm{I}^{(\ell)}$, whose $i^{\textit{th}}$ component, denoted by $\bm{I}^{(\ell)}(x_i)$, is the number of $\ell$-length attackers or defenders of $x_i$. It has been proved by \cite{ref-west2001introduction} that $\bm{I}^{(\ell)}$ can be computed by the calculation of the $\ell^{\textit{th}}$ power of $\bm{A}$, i.e.,
\begin{equation}\label{Eqn_CountingLength}
  \bm{I}^{(\ell)} = \bm{A}^{\ell}\bm{e}, ~ \ell = 0,1,2,\cdots
\end{equation}
where $\bm{e}$ is the column vector consisting of all ones. Given the maximum walk  length $k$ (which will be used in order to capture finite attackers and defenders under $k$), then the simple counting model can be obtained by positively summing $\bm{I}^{(\ell)}$ for all even $\ell$ under $k$ and negatively summing $\bm{I}^{(\ell)}$ for all odd $\ell$ under $k$, i.e.,
\begin{equation} \label{Eqn_SimpleCounting}
  \bm{v}^{(k)}=\sum^{k}_{\ell=0}(-1)^\ell \bm{I}^{(\ell)} = \sum^{k}_{\ell=0}(-1)^\ell\bm{A}^\ell\bm{e}
\end{equation}
in which $\bm{v}^{(k)}$ is the vector of strength values assigning to arguments in $\mathcal{X}$, and the $i^{\textit{th}}$ component of $\bm{v}^{(k)}$ is the strength values of $x_i$, denoted by $\bm{v}^{(k)}(x_i)$. The greater $k$, the closer the evaluation obtained to the actual counting approach. As $k$ goes to $\infty$, $\bm{v}^{(k)}$ is the evaluation on arguments.

\begin{example} \label{Exp_SimpleCounting}
The attack matrix of the argumentation framework in Fig.~\ref{Fig_AAF} is
\begin{equation*}
  \bm{A}=
  \left[
    \begin{array}{cccc}
      0 & 1 & 0 & 0 \\
       0  &  0  &  1  &  1  \\
       0  &  1  &  1  &  0  \\
       0  &  0  &  0  &  0  \\
    \end{array}
  \right]
\end{equation*}
Then, we can obtain
\begin{eqnarray*}
  \bm{I}^{(0)} &=& \bm{A}^0\bm{e}= [1, 1, 1, 1]^T \\
  \bm{I}^{(1)} &=& \bm{A}^1\bm{e}= [1, 2, 2, 0]^T \\
  \bm{I}^{(2)} &=& \bm{A}^2\bm{e}= [2, 2, 4, 0]^T \\
  \cdots
\end{eqnarray*}
All entries of $\bm{I}^{(0)}$ are $1$s since each argument is a $0$-length defender of itself. $\bm{I}^{(1)}$ indicates that $x_1$ has one $1$-length attacker ($x_2 \rightarrow x_1$), $x_2$ has two ($x_3 \rightarrow x_2$ and $x_4 \rightarrow x_2$), $x_3$ has two ($x_2 \rightarrow x_3$ and $x_2 \rightarrow x_3$), and $x_4$ has zero. $\cdots$. Now, the calculations of the counting approach under $k$ ($k=0,1,2,\cdots$) are
\begin{eqnarray*}
  \bm{v}^{(0)} &=& \bm{I}^{(0)}= [1, 1, 1, 1]^T \\
  \bm{v}^{(1)} &=& \bm{I}^{(0)} - \bm{I}^{(1)}= [0, -1, -1, 1]^T \\
  \bm{v}^{(2)} &=& \bm{I}^{(0)} - \bm{I}^{(1)} + \bm{I}^{(2)}= [2, 1, 3, 1]^T \\
  \cdots
\end{eqnarray*}
\end{example}

However, this simple model leads to two problems. The first is that for an attack graph with cycles, when $k$ goes to $\infty$, then some arguments may have infinite number of attackers and defenders. This is not conducive to comparison and practical application since if the counting values of two arguments are both infinite, we can not compare them.

The second problem is that the simple counting model does not distinguish different lengths of attackers and defenders. Different lengths of attackers or defenders of an argument may have different impacts on the argument. The simple model just simply count them together and does not consider which is more important and which is less important. This problem may lead to some counter intuitions. Considering Example~\ref{Exp_SimpleCounting}, for instance, $\bm{v}^{(2)}(x_3)=3$ is greater than $\bm{v}^{(2)}(x_4)=1$. It is counterintuitive since $x_4$ is not attacked and should be most acceptable. We consider that shorter attackers and defenders are preferred, which can effectively drive the agents to make only relevant moves, and thus we assume that a shorter attacker (respectively, defender) of an argument has more effect than a longer one on the argument \cite{ref-rienstra2013opponent}.

To remedy these two problems, we define a \emph{normalization factor}, which can ensure that the argument strength scale is bounded, and a \emph{damping factor} on walk length, which allows a more refined treatment on different length of attackers and defenders. Now, we redefine the vector $\bm{I}^{(\ell)}$ as
\begin{equation}\label{Eqn_ImprovedCountingLength}
  \bm{I}^{(\ell)}_\alpha = \alpha^\ell\widetilde{\bm{A}}^\ell\bm{e}
\end{equation}
in which $\alpha\in(0,1)$ is the damping factor and $\widetilde{\bm{A}}$ is the normalized attack matrix defined as $\widetilde{\bm{A}}=\bm{A}/{N}$ where the scalar $N  = \|\bm{A}\|_\infty=\max_{1\leq i \leq n}\sum_{j=1}^{n}|a_{ij}|$ is the normalization factor.\footnote{The normalization factor $N$ we used is the infinity norm of the attack matrix (see \cite{ref-horn2012matrix}). It is dynamic and represents the ``size'' of the argumentation framework.} It can be seen that the damping factor $\alpha$ provides a graded treatment of attackers and defenders of various lengths since the longer the walk length $\ell$, the smaller the $\alpha^\ell$. Then, the improved counting model is defined as
\begin{equation} \label{Eqn_ImprovedCounting}
  \bm{v}^{(k)}_\alpha= \sum^{k}_{\ell=0}(-1)^\ell\bm{I}^{(\ell)}_\alpha = \sum^{k}_{\ell=0}(-1)^\ell\alpha^\ell\widetilde{\bm{A}}^\ell\bm{e}
\end{equation}
We have shown that the improved counting model can range the strength value of each argument into the interval $[0,1]$ (see \cite[Thm~1]{ref-pu2015counting}) and converge to a unique solution in $[0,1]$ as $k$ goes $\infty$ (see \cite[Thm~2]{ref-pu2015counting}). Then, we can define the counting semantics for an argumentation framework as the limit of $\{\bm{v}^{(k)}_\alpha\}^\infty_{k=0}$.
\begin{definition}
\label{Def_ArgStrength}
Let $\textit{AF}=\left< \mathcal{X}, \mathcal{R}\right>$ be an argumentation framework with $\mathcal{X}=\{x_1,x_2,\cdots,x_n\}$, and the damping factor $\alpha\in(0,1)$. The \emph{attacker and defender counting semantics} for such $\textit{AF}$ is, for all arguments $\mathcal{X}$,
\begin{equation*}
  \bm{v}_\alpha = \lim_{k\rightarrow\infty} \bm{v}^{(k)}_\alpha
\end{equation*}
The strength value of $x_i$ is the $i^{\textit{th}}$ component of $\bm{v}_\alpha$, denoted by $\bm{v}_\alpha(x_i)$.
\end{definition}
The counting semantics can be approximated iteratively by
\begin{equation} \label{Eqn_Iteration}
  \bm{v}^{(k)}_\alpha = \bm{e} - \alpha\widetilde{\bm{A}}\bm{v}^{(k-1)}_\alpha
\end{equation}
with the initial valuation $\bm{v}^{(0)}_\alpha=\bm{e}$ under a given tolerance $\epsilon$ (i.e., the iteration terminates when the change $\|\bm{v}^{(k)}_\alpha-\bm{v}^{(k-1)}_\alpha\| \leq \epsilon$).

\begin{example} \label{Exp_ImprovedCounting}
We continue Example~\ref{Exp_SimpleCounting}. The normalization factor of $\bm{A}$ is $N=2$, thus the normalized attack matrix is
\begin{equation*}
  \widetilde{\bm{A}}=
  \left[
    \begin{array}{cccc}
      0 & 1 / 2 & 0 & 0 \\
       0  &  0  &  1 / 2  &  1 / 2  \\
       0  &  1 / 2  &  1 / 2  &  0  \\
       0  &  0  &  0  &  0  \\
    \end{array}
  \right]
\end{equation*}
Assume $\alpha = 0.98$, then the counting values of each argument are summarized below:
\begin{align*}
  \bm{v}^{(0)}_\alpha &= \bm{e} = [1.00, 1.00, 1.00, 1.00]^T \\
  \bm{v}^{(1)}_\alpha &= \bm{e} - \alpha\widetilde{\bm{A}}\bm{v}^{(0)}_\alpha = [0.51, 0.02, 0.02, 1.00]^T \\
  \bm{v}^{(2)}_\alpha &= \bm{e} - \alpha\widetilde{\bm{A}}\bm{v}^{(1)}_\alpha = [0.99, 0.50, 0.98, 1.00]^T \\
  \cdots
\end{align*}
If we set $\epsilon=10^{-3}$, after finitely many iterations, the counting values gradually tends to be stable and converges to the approximative counting semantics $\bm{v}_\alpha = [0.89,0.22,0.60,1.00]^T$.
\end{example}

In the following sections, we will continue our previous work in \cite{ref-pu2015counting} and discuss how to determine the damping factor $\alpha$. Moreover, we relate our proposal with some existing approaches.

\section{The determination of the damping factor} \label{Sec_DampingFactor}
The damping factor plays an important role in the counting semantics. It not only provides a more refined treatment on different length of attackers and defenders but also controls the convergence speed of the computation. How to choose the damping factor $\alpha$ in practical application thus is an important question for the counting semantics.

Before this, let us see how the damping factor impact the strength values of arguments. It can be seen from Eqn.~\eqref{Eqn_ImprovedCountingLength} that the impact of the $\ell$-length attackers or defenders on arguments satisfies
\begin{equation*}
  \bm{I}^{(\ell)}_\alpha \leq \alpha^\ell\bm{e}
\end{equation*}
Obviously, the lower the $\alpha$, the shorter the $\ell$ causing the $\bm{I}^{(\ell)}_\alpha$ approaching to zero. In other words, the lower the $\alpha$, the less number of attackers and defenders are considered to contribute to the strength values of arguments since the attacker or defenders whose length greater than some length $\ell^\ast$ may have a little impacts on the strength values and can be ignored. However, we expect a ``ideal'' counting semantics, i.e., considering attackers and defenders as much as possible, hence we are inclined to chose $\alpha$ as close to $1$ as possible.

Another reason to choose a great $\alpha$ is that the lower bound of the counting semantics is $1-\alpha$ (i.e., $\bm{v}_\alpha(x_i)\geq 1-\alpha$ for any $x_i\in\mathcal{X}$). When $\alpha$ is close to $0$, the lower bound is close to $1$, resulting in a trivial uniform assessment, which is away from the goal of the comparisons among arguments.

However, as $\alpha$ approaches to $1$, more time may be needed to make  Eqn.~\eqref{Eqn_Iteration} convergence. We thus needs a tradeoff between the performance and the ``ideal'' semantics. By Eqn~\eqref{Eqn_Iteration}, it can be proved that the change $\delta^{(k)}$, at the $k^{\textit{th}}$ iteration, satisfies
\begin{equation}\label{Eqn_MeanChange}
  \delta^{(k)} = \|\bm{v}_\alpha^{(k)}-\bm{v}_\alpha^{(k-1)}\| \leq [\alpha\rho(\bm{\widetilde{A}})]^k
\end{equation}
where $\rho(\bm{\widetilde{A}})$ is the spectral radius of $\bm{\widetilde{A}}$ and $\rho(\bm{\widetilde{A}}) \leq \|\bm{\widetilde{A}}\|_\infty = 1$. Accordingly, we can conclude that the convergence rate of the computation of the counting semantics is the rate at which $[\alpha\rho(\bm{\widetilde{A}})]^k \rightarrow 0$. If we expect the counting semantics to converge to a tolerance level $\epsilon$ with at most $k_{\max}$ iterations, then we have the following rough estimate:
\begin{equation}\label{Eqn_RoughEstimate}
  k_{\max}=\frac{\log_{10} \epsilon}{\log_{10} [ \alpha\rho(\bm{\widetilde{A}})]}
\end{equation}
Assume $\rho(\bm{\widetilde{A}})= 1$, for $\epsilon=10^{-5}$ and $\alpha=0.97$, one can expect at most $\frac{-5}{\log_{10} 0.97}\thickapprox 378$ iterations until convergence to the counting semantics. For $\alpha = 0.98$, about $570$ iterations and for $\alpha=0.99$, about $1146$ iterations, as shown in Figure~\ref{Fig_DampingFactor}. If we expect the counting semantics to converge using no more than $500$ iterations, then $\alpha$ should be less than $0.98$. In most applications, however, $\rho(\bm{\widetilde{A}})$ is always far less than one (since the attack matrix $\bm{A}$ is often sparse), implying that $\alpha$ can be much larger. Therefore, we usually choose $\alpha$ in $[0.90,0.98]$.

\begin{figure}[htb]
\centering
  \includegraphics[width=.44\textwidth]{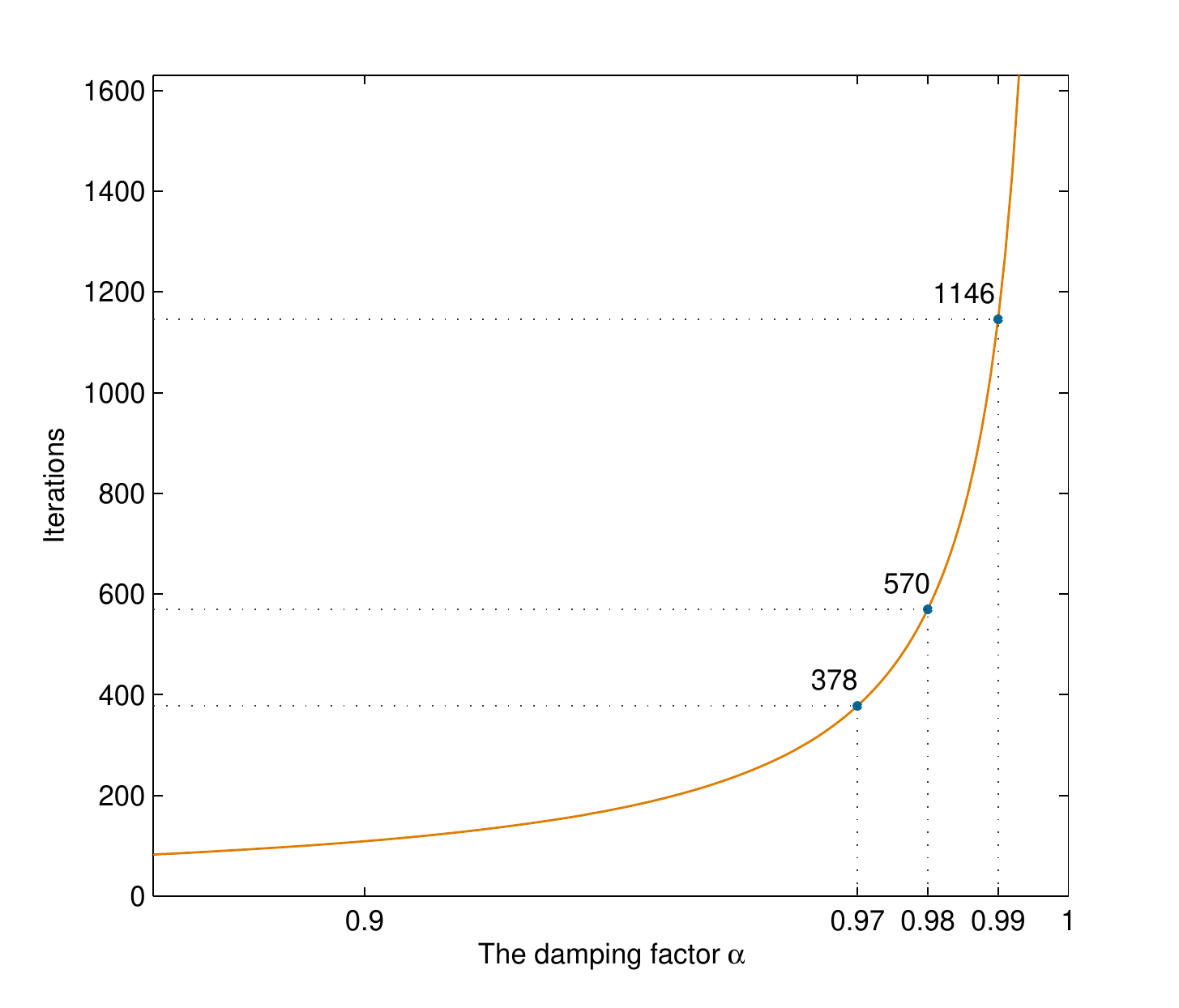}
\caption{The relation between iterations and the damping factor.}
\label{Fig_DampingFactor}
\end{figure}

\section{Counting Semantics vs. Dung's Semantics} \label{Sec_Comparation_Classical}
In this section, we investigate the relationships between Dung's classical semantics and the counting semantics in more depth. All of Dung's classical semantics are mainly grounded on set theory, while our counting semantics is based on numerical matrix operations. It seems that these two kinds of semantics have nothing to do with each other. In this paper, however, we will build their relationships by boolean algebra.

\subsection{Solve Classical semantics by Boolean Algebra}
To begin with, let us introduce two operations on boolean matrices.
\begin{definition}[Boolean Operations]
Let $\bm{A}_{m\times k}$ and $\bm{B}_{k\times n}$ be two boolean matrices (i.e., all entries are either logic $1$ or logic $0$).
\begin{itemize}
  \item The \emph{boolean product} of $\bm{A}$ and $\bm{B}$, denoted by $\bm{A}\odot\bm{B}$, is $\bm{C}_{m\times n}$ defined by
\begin{equation} \label{Eqn_BoolProduct}
  c_{ij}=\bigvee_{h=1}^{k}(a_{ih}\wedge b_{hj})
\end{equation}
  \item The \emph{negation} of $\bm{A}$, denoted by $\neg \bm{A}$, is a cellwise operator and defined by
\begin{equation} \label{Eqn_BoolNegation}
  [\neg \bm{A}]_{ij}=\neg a_{ij}
\end{equation}
\end{itemize}
\end{definition}

Moreover, let us introduce the boolean vector representation of a subset of a set of arguments and the boolean matrix representation of the attack relations between the arguments. Assume $\mathcal{X}=\{x_1,x_2,\cdots,x_n\}$ is a set of arguments, then $S\subseteq \mathcal{X}$ can be encoded by $n \times 1$ boolean column vector $\mathbf{g}_S$, whose $i^{\textit{th}}$ component $[\mathbf{g}_S]_i$ is $1$ if $x_i\in S$; otherwise $0$. Clearly, if $S=\emptyset$ then all components of $\mathbf{g}_S$ are $0$s, i.e., $\mathbf{g}_S=\mathbf{0}$, and if $S=\mathcal{X}$ then $\mathbf{g}_S=\bm{e}$. Intuitively, $\mathbf{g}_{\overline{S}}=\neg \mathbf{g}_S$ where $\overline{S}$ is the complement of $S$ with respect to $\mathcal{X}$, i.e., $\overline{S} =\mathcal{X}\backslash S$. We can utilize the attack matrix $\bm{A}$ to represent the attack relations between arguments when all the entries $\bm{A}$ are considered to be logic $0$ or $1$.

In the following, we will present the boolean algebra approaches to solve Dung's set-theory-based semantics.

\begin{theorem} \label{Prop_R+R-}
  Let $\textit{AF}=\left< \mathcal{X}, \mathcal{R}\right>$ with $\mathcal{X}=\{x_1,x_2,\cdots,x_n\}$, and its attack matrix be $\bm{A}$. Assume $S \subseteq \mathcal{X}$, then {\rm (i).} $\mathbf{g}_{\mathcal{R}^+(S)}=\bm{A}\odot\mathbf{g}_S$; {\rm (ii).} $\mathbf{g}_{\mathcal{R}^-(S)}=\bm{A}^T\odot\mathbf{g}_S$.
\end{theorem}
\begin{proof}
  (i). Two cases need to be consider: \\
  \textbf{Case 1:} For any $x_i\in \mathcal{R}^+(S)$, i.e. $[\mathbf{g}_{\mathcal{R}^+(S)}]_i=1$, there must exist an argument $x_j\in S$ such that $x_j\mathcal{R}x_i$. This follows that $[\mathbf{g}_S]_j=1$ and $a_{ij}=1$. Hence,
  \begin{align*}
    [\bm{A}\odot\mathbf{g}_S]_i &= \bigvee^n_{h=1}(a_{ih}\wedge [\mathbf{g}_S]_h) =a_{ij}\wedge[\mathbf{g}_S]_j=1=[\mathbf{g}_{\mathcal{R}^+(S)}]_i
  \end{align*}
  \textbf{Case 2:} For any $x_i \notin \mathcal{R}^+(S)$, namely $[\mathbf{g}_{\mathcal{R}^+(S)}]_i=0$, then $\forall x_j\in S$ such that $\nexists x_j\mathcal{R}x_i$, i.e. for all $[\mathbf{g}_S]_j=1$ such that $d_{ij}=0$. On the other hand, if $x_j \notin S$ then $[\mathbf{g}_S]_j=0$. Thus for all $j \in \{1,2,\cdots,n\}$, $a_{ij}\wedge[\mathbf{g}_S]_j=0$. It implies that
  \begin{equation*}
    [\bm{A}\odot\mathbf{g}_S]_i=\bigvee^n_{j=1}0=0=[\mathbf{g}_{\mathcal{R}^+(S)}]_i
  \end{equation*}
  To sum up the above proof, $[\bm{A}\odot\mathbf{g}_S]_i=[\mathbf{g}_{\mathcal{R}^+(S)}]_i$ always holds for any argument $x_i$, and this concludes the theorem.  \\
  (ii). Similar to the proof of (i).
\end{proof}

Clearly, this theorem provides a way to define the boolean representation of function $\mathcal{R}^+$ and $\mathcal{R}^-$, i.e. $\mathcal{R}^+(\mathbf{g})=\bm{A}\odot\mathbf{g}$ and $\mathcal{R}^-(\mathbf{g})=\bm{A}^T\odot\mathbf{g}$. It is mentioned in \cite{ref-AC98} that the characteristic function and the operator $\mathcal{R}^+$ has the relation as below:
\begin{equation}\label{Eqn_CharacteristicFunction}
  \mathfrak{F}(S) = \overline{\mathcal{R}^+(\overline{\mathcal{R}^+(S)})}
\end{equation}
We can then rewrite the boolean representation of the characteristic function $\mathfrak{F}$ as:
\begin{equation}\label{Eqn_Boolean_CF}
  \mathfrak{F}(\mathbf{g})=\neg \mathcal{R}^+\big(\neg \mathcal{R}^+(\mathbf{g})\big)=\neg \big(\bm{A} \odot \neg(\bm{A}\odot \mathbf{g})\big)
\end{equation}
Base on this, another iteration approach for computing grounded extension can be established from the initial value $\mathbf{g}^{(0)}=\mathbf{0}$ (i.e. the empty set),
\begin{equation}\label{Eqn_BOOL_Iteration}
  \mathbf{g}^{(k)} = \mathfrak{F}(\mathbf{g}^{(k-1)})= \neg \big(\bm{A} \odot \neg(\bm{A}\odot \mathbf{g}^{(k-1)})\big)
\end{equation}

\subsection{The Relationships Between The Counting Semantics and Dung's Semantics}
Now, we move our attention to the relationship between the counting semantics and Dung's classical semantics. In order to facilitate the comparison and analysis, we convert the boolean operations in Eqn.~\eqref{Eqn_BOOL_Iteration} into arithmetic operations by employing the following identities\footnote{In these identities, boolean $0$ and $1$ are interpreted as integers.}:
\begin{align}
  \neg \mathbf{g} &= \bm{e} - \mathbf{g} \label{Eqn_BOOL_NEG}  \\
  \bm{A} \odot \mathbf{g} &= \sgn(\bm{A}\cdot \mathbf{g}) \label{Eqn_BOOL_MatrixVector}
\end{align}
in which the $\sgn$ is the signum function, defined as $\sgn(x)=1$ if $x>0$ and $\sgn(x)=0$ when $x=0$. The $\sgn$ of a matrix means performing the $\sgn$ operation on each cell of this matrix.


Then Eqn.~\eqref{Eqn_BOOL_Iteration} can be represented as
\begin{align}
   & \mathbf{g}^{(k)} = \bm{e} - \bm{A} \odot \big(\bm{e}-\bm{A}\odot \mathbf{g}^{(k-1)}\big) \label{Eqn_ArithIter0} \\
  \Leftrightarrow~~ & \mathbf{g}^{(k)} = \bm{e} - \sgn\big\{\bm{A} \cdot \big[\bm{e}-\sgn(\bm{A}\cdot \mathbf{g}^{(k-1)})\big]\big\} \label{Eqn_ArithIter1} \\
  \Leftrightarrow~~ & \mathbf{g}^{(k)} = \bm{e} - \sgn\big\{\widehat{\bm
   {A}}\cdot \big[\bm{e}-\sgn(\widehat{\bm
   {A}}\cdot \mathbf{g}^{(k-1)})\big]\big\} \label{Eqn_ArithIter2}
\end{align}
where $\widehat{\bm{A}} = \alpha\widetilde{\bm{A}}$. On the other hand, we rewrite the iterative formula of the counting semantics, i.e. Eqn.~\eqref{Eqn_Iteration}, as
\begin{equation}\label{Eqn_GT_Iter}
  \bm{v}_\alpha^{(k)}=\bm{e} - \widehat{\bm{A}}\cdot(\bm{e}-\widehat{\bm{A}}\cdot\bm{v}_\alpha^{(k-2)})
\end{equation}

Now, let us define another function $\mathfrak{G}: 2^\mathcal{X} \mapsto 2^\mathcal{X}$ such that for $S\subseteq \mathcal{X}$, $\mathfrak{G}(S)=\{x \in \mathcal{X}|x ~\mbox{is not attacked by}~ S ~\mbox{w.r.t the relation}~ \mathcal{R}\}$, alternately,
\begin{equation*}
\mathfrak{G}(S) = \overline{\mathcal{R}^+(S)}
\end{equation*}
Obviously, $\mathfrak{F}(S) = \mathfrak{G}\circ\mathfrak{G}(S)$ and a conflict-free set $S\subseteq \mathcal{X}$ is a stable extension iff $S$ is a fixed point of $\mathfrak{G}$. With the similar idea above, we write the boolean representation of $\mathfrak{G}$ as
\begin{equation}\label{Eqn_BL_STABLE}
  \mathfrak{G}(\mathbf{g}) = \neg(\bm{A}\odot \mathbf{g}) = \bm{e}-\sgn(\widehat{\bm
   {A}}\cdot \mathbf{g})
\end{equation}

Until now, we can easily observe that the iteration formulas of Eqn.~\eqref{Eqn_ArithIter2} and Eqn.~\eqref{Eqn_GT_Iter} have similar iteration structures apart from the $\sgn$ operations as well as the Eqn.~\eqref{Eqn_Iteration} and Eqn.~\eqref{Eqn_BL_STABLE}. This establishes quite interesting relationships between our counting semantics and classical semantics: Firstly, both the counting semantics and classical semantics are interaction-based valuation since here the value of each argument merely relies on the graph structure of the argumentation framework. Secondly, by Definition~\ref{Def_Acceptability}, it can be seen that the complete, grounded and preferred extension are included in the set of the fixed points of the Eqn.~\eqref{Eqn_ArithIter2} and the stable extension is the fixed point of the Eqn.~\eqref{Eqn_BL_STABLE}, while the counting semantics is the fixed point of Eqn.~\eqref{Eqn_GT_Iter}. Lastly, both the grounded semantics and the counting semantics (given the damping factor $\alpha$) are unique.

However, the main difference is that our counting semantics is more general since it allows for a more finer level between two extremes views (accepted and rejected) on reasoning, while the classical semantics represents an extreme case since it just involves being accepted or rejected. As a result, these classical semantics may cause empty extension. This might be suitable for reasoning but not for practical applications in some scenarios. Considering an argument system whose grounded extension is empty, for example, if a decision must be made, then the grounded semantics is unavailable since all arguments are unacceptable in this case. Our proposal engages in comparing and ranking arguments from the most acceptable to the weakest one(s). It is possible that the most acceptable argument(s) might be the good choice for some problems.

Another difference is that some classical semantics such as complete and preferred semantics may provide multiple extensions. This may cause that (with a credulous perspective) the recommendation is to administer almost all decisions. However, our counting semantics can always return a unique solution (specifying the damping factor $\alpha$). If one expects to achieve the multiplicity property from the counting semantics, he or she can choose a different damping factor $\alpha$ since various $\alpha$ may give different rankings on arguments.

\section{Relating with Generic Gradual Valuations}  \label{Sec_Comparation_Gradual}
Generic Gradual Valuations are one of the two proposals for introducing graduality in the interaction-based valuation, namely a local approach and a global approach \cite{ref-cayrol2005graduality}. It is local because strength values are assigned to arguments based solely on their direct attackers (so defenders are also taken into account through the attackers). In global approach, however, the value of an argument represents the set of all the attack and all the defence branches of different lengths for this argument. Overall, generic gradual valuations are a very versatile and expressive framework for assigning labels to arguments, where the labels indicate a set of possible acceptability or truth values to each argument. In this section, we will give comparisons between our approach and these two approaches.

\subsection{Comparison with Global Valuation}
At the first sight, the global approach is similar to ours. It assigns a pair of tuples to each argument: the first tuple contains the lengths of the defence branches (even integers), the other one contains the lengths of the attack branches (odd integers). This approach is claimed to be ``global'' since it computes the value of the argument using the whole attack graph influencing the argument.

However, there are three main differences between our approach and tupled approach. First, we employ vectors to memorize the number of attackers and defenders of different lengths. Second, to compute the number of attackers and defenders, we utilize a matrix approach, i.e., computing powers of attack matrix. While in tupled approach, a propagation algorithm is proposed to compute the tupled values. Third, to determine which argument is more acceptable, in tupled approach, a double comparison (first quantitative, then qualitative) algorithm is presented based on classical lexicographic comparison. One drawback of the tupled approach is that incomparable tupled values may exist. While our proposal is a purely quantitative mathematical approach, which assigns a numerical value in $[0,1]$ to each argument and facilitates the comparison between arguments.

\subsection{Comparison with Local Valuations}
The local approach is a very generic formalisation capable of capturing the discrete labeling semantics of \cite{ref-jakobovits1999robust} and the continuous ``categoriser'' function of \cite{ref-besnard2001logic}. Assume a totally ordered set $W$ with a minimum element $V_{\min}$ and a subset $V$ of $W$ containing $V_{\min}$ and a top element $V_{\max}$. Let $\textit{AF}=\left< \mathcal{X}, \mathcal{R}\right>$ with $\mathcal{X}=\{x_1,\cdots,x_n\}$. A local valuation is a function $v: \mathcal{X}\mapsto V$ such that:
\begin{itemize}
  \item $v(a)\geq V_{\min}$, $\forall a\in\mathcal{X}$.
  \item $v(a)=V_{\max}$, if $\mathcal{R}^-(a)=\emptyset$, $\forall a \in \mathcal{X}$.
  \item $v(a)=g\left(h(v(a_1), \cdots, v(a_m))\right)$, if $\mathcal{R}^-(a)=\{a_1, \cdots, a_m\} \neq \emptyset$.
\end{itemize}
where function $g: W\mapsto V$ such that $g$ is non-increasing, $g(V_{\min})=V_{\max}$ and $g(V_{\max})<V_{\max}$, and function $h: V^\ast\mapsto W$ such that ($V^\ast$ denotes the set of all finite sequences of elements of $V$)
\begin{itemize}
  \item $h(t)=t$ and $h()=V_{\min}$.
  \item $h(t_1,\cdots,t_m,t_{m+1})\geq h(t_1,\cdots,t_m)$.
  \item $h(t_1,\cdots,t_i,\cdots,t_m)\geq h(t_1,\cdots,t'_i,\cdots,t_m)$ where $t_i\geq t'_i$.
  \item $h(t_{i1},\cdots,t_{im}) = h(t_{1},\cdots,t_{m})$ for any permutation $\{t_{i1},\cdots,t_{im}\}$ of $\{t_{1},\cdots,t_{m}\}$.
\end{itemize}

By Eqn.~\eqref{Eqn_Iteration}, our proposal can be defined as an instance of the generic valuation such that
\begin{itemize}
  \item $V=[0,1]$, $W=[0,\infty[$, $V_{\min}=0$ and $V_{\max}=1$.
  \item $g(t) = 1-\frac{\alpha}{N}t$
  \item $h(t_1,\cdots,t_m)=t_1+\cdots+t_m$
\end{itemize}
Actually, we define $h: V^{n}\mapsto W$ as $h(\bm{t})=\bm{A}_{i\ast}\bm{t}$ where $n$ is the number of arguments in $\textit{AF}$, $\bm{A}_{i\ast}$ is the $i^{\textit{th}}$ row vector of $\bm{A}$ and $\bm{t}=[t_1, \cdots, t_n]^T$ is a $n$-dimensional column vector. Clearly, the two definitions are equivalent.

The main difference between our proposal and categoriser function is that the latter defines $g(t)=\frac{1}{1+t}$. Clearly, it can be seen that both approaches formalise the same intuition that an argument with low-valued attackers obtains high values and an argument with high-valued attackers obtains low values. Furthermore, both approaches are continuous and evaluate the strength of arguments on a scale of numerical values from $0$ to $1$.

It has proved that there is a bijective correspondence between labeling-based and extension-based semantics for complete, grounded and preferred semantics \cite{ref-caminada2006issue}. Therefore, the comparisons between our proposal and labeling semantics can refer to Section~\ref{Sec_Comparation_Classical}.

\section{Ranking-based Axiomatic Perspective On The Counting Semantics} \label{Sec_AxiomOnArgRank}
The counting semantics assign arguments a vector of numerical strength values, which are relative and do not make sense when they are not compared with each other. Actually, in most applications, we merely concern the ranking (ordering) over arguments induced by the counting semantics. Given the damping factor $\alpha$, the \emph{ranking} $\succeq_\alpha$ on the set of arguments $\mathcal{X}$ derived from the counting semantics $\bm{v}_\alpha$ is defined by: for any $x,y\in\mathcal{X}$, $x\succeq_\alpha y$ iff $\bm{v}_\alpha(x)\geq\bm{v}_\alpha(y)$. Intuitively, $\succeq_\alpha$ is total (i.e., $\forall x,y\in\mathcal{X}$, $x\succeq_\alpha y$ or $y\succeq_\alpha x$) and transitive (i.e., $\forall x,y,z\in\mathcal{X}$, if $x\succeq_\alpha y$ and $y\succeq_\alpha z$, then $x\succeq_\alpha z$). Note that here $x\succeq_\alpha y$ means that argument $x$ is at least as acceptable as argument $y$ with respect to $\alpha$. Formally, we define $x \simeq_\alpha y$ if and only if $x\succeq_\alpha y$ and $y\succeq_\alpha x$, which means $x$ and $y$ are equally acceptable w.r.t. $\alpha$. Moreover, $x \succ_\alpha y$, meaning $x$ is strictly more acceptable than $y$ w.r.t. $\alpha$, if and only if $x \succeq_\alpha y$ but not $y \succeq_\alpha x$.

In \cite{ref-amgoud2013ranking}, the authors propose a set of axioms (postulates), each of which represents a criterion, and is an intuitive and desirable property that a ranking-based semantics may enjoy. Such an axiomatic approach empowers a better understanding of semantics and a more accurate comparison between different proposals. In this section, we will formally show that the ranking $\succeq_\alpha$ derived from the counting semantics satisfies some of these postulates.

The first axiom is that a ranking on a set of arguments does not depend on their identity but merely on the attack relations among them. In other words, if two argumentation system are isomorphic, then they are equivalent and should have the same ranking semantics. The isomorphisms between two argumentation frameworks $\textit{AF}_1=\left< \mathcal{X}_1, \mathcal{R}_1\right>$ and $\textit{AF}_2=\left< \mathcal{X}_2, \mathcal{R}_2\right>$ is a bijective function $\tau$: $\mathcal{X}_1 \mapsto \mathcal{X}_2$ such that for all $x,y\in \mathcal{X}_1$, $x\mathcal{R}_1 y$ if and only if $\tau(x)\mathcal{R}_2\, \tau(y)$. Now we define the first axiom as follows:
\begin{axiom}[Abstraction {\sf (Ab)}]
A ranking-based semantics $\mathrm{\Gamma}$ satisfies \emph{abstraction} iff for any two argumentation frameworks $\textit{AF}_1=\left< \mathcal{X}_1, \mathcal{R}_1\right>$ and $\textit{AF}_2=\left< \mathcal{X}_2, \mathcal{R}_2\right>$, for every isomorphism $\tau$ from $\textit{AF}_1$ to $\textit{AF}_2$, we have that $\forall x,y\in \mathcal{X}_1$, $x \succeq^{\textit{AF}_1}_{\mathrm{\Gamma}} y$ iff $\tau(x) \succeq^{\textit{AF}_2}_{\mathrm{\Gamma}} \tau(y)$.
\end{axiom}

The second axiom states that the question whether argument $x$ is ranked higher than argument $y$ should be independent of any argument $z$ that is not connected to $x$ or $y$, i.e., there is no path from $x$ or $y$ to $z$ (neglecting the direction of the edges). Let $\mathcal{C}(\textit{AF})$ be the set of weak connected components of $\textit{AF}$. Each weak connected component of $\textit{AF}$ is a maximal subgraph of $\textit{AF}$ in which any two arguments are mutually connected by a path (neglecting the direction of the edges).
\begin{axiom}[Independence {\sf (In)}]
A ranked-based semantics $\mathrm{\Gamma}$ satisfies \emph{independence} iff for any argumentation framework $\textit{AF}=\left< \mathcal{X}, \mathcal{R}\right>$, and for any $\textit{AF}_c=\left< \mathcal{X}_c, \mathcal{R}_c\right>$ such that $\textit{AF}_c \in 2^{\mathcal{C}(\textit{AF})}$, $\forall x,y\in\mathcal{X}_c$, $x \succeq^{\textit{AF}}_{\mathrm{\Gamma}} y$ iff $x \succeq^{\textit{AF}_c}_{\mathrm{\Gamma}} y$.
\end{axiom}

The third axiom encodes the idea that non-attacked arguments are more acceptable (and thus ranked higher) than attacked ones.
\begin{axiom}[Void Precedence {\sf (VP)}]
A ranked-based semantics $\mathrm{\Gamma}$ satisfies \emph{void precedence} iff for any $\textit{AF}=\left< \mathcal{X}, \mathcal{R}\right>$, $\forall x,y\in\mathcal{X}$, if $\mathcal{R}^-(x)=\emptyset$ and $\mathcal{R}^-(y)\neq\emptyset$ then $x \succ^{\textit{AF}}_{\mathrm{\Gamma}} y$.
\end{axiom}

The fourth postulate states that having attacked attackers is more acceptable than non-attacked attackers, i.e., being defended is ranked higher than not being defended. For any $x\in\mathcal{X}$, we denote by $\mathcal{D}(x)=\{y\in\mathcal{X}| \exists z\in\mathcal{X} \hbox{~such that~} z\mathcal{R}x \hbox{~and~} y\mathcal{R}z\}$ the set of all defenders of argument $x$ in $\left< \mathcal{X}, \mathcal{R}\right>$. Alternatively, we can also write $\mathcal{D}(x)=\{y\in\mathcal{X}|y \in \mathcal{R}^-\big(\mathcal{R}^-(x)\big)\}$.
\begin{axiom}[Defense precedence {\sf (DP)}]
A ranked-based semantics $\mathrm{\Gamma}$ satisfies \emph{defense precedence} iff for every $\textit{AF}=\left< \mathcal{X}, \mathcal{R}\right>$, $\forall x,y\in\mathcal{X}$, if $|\mathcal{R}^-(x)|=|\mathcal{R}^-(y)|$, $\mathcal{D}(x)\neq \emptyset$, and $\mathcal{D}(y)=\emptyset$ then $x \succ^{\textit{AF}}_{\mathrm{\Gamma}} y$.
\end{axiom}

The next axiom says that an argument $x$ should be at least as acceptable as argument $y$, when the direct attackers of $y$ are at least as numerous and well-ranked as those of $x$. This involves a relation that compares sets of arguments, i.e., \emph{group comparison}: Let $\succeq$ be a ranking on a set of arguments $\mathcal{X}$. For any argument subset $S_1,S_2\subseteq \mathcal{X}$, $S_1 \succeq S_2$ iff there exists an injective mapping $\delta$ from $S_2$ to $S_1$ such that $\forall x\in S_2$, $\delta(x)\succeq x$. Obviously, if $S_1 \succeq S_2$, then there must be $|S_1|\geq |S_2|$ and for arbitrary $y \in S_2$, $\exists x \in S_1$ such that $x \succeq y$.
\begin{axiom}[Counter-Transitivity {\sf (CT)}]
A ranked-based semantics $\mathrm{\Gamma}$ satisfies \emph{counter-transitivity} iff for every $\textit{AF}=\left< \mathcal{X}, \mathcal{R}\right>$, $\forall x,y\in\mathcal{X}$, if $\mathcal{R}^-(y) \succeq^{\textit{AF}}_{\mathrm{\Gamma}} \mathcal{R}^-(x)$ then $x \succeq^{\textit{AF}}_{\mathrm{\Gamma}} y$.
\end{axiom}

Moreover, $S_1 \succ S_2$ is a \textit{strict group comparison} iff it satisfies two conditions: (1) $S_1 \succeq S_2$; (2) $|S_1|>|S_2|$ or $\exists x \in S_2$, $\delta(x)\succ x$.
\begin{axiom}[Strict Counter-Transitivity {\sf (SCT)}]
A ranked-based semantics $\mathrm{\Gamma}$ satisfies \emph{strict counter-transitivity} iff for any $\textit{AF}=\left< \mathcal{X}, \mathcal{R}\right>$, $\forall x,y\in\mathcal{X}$, if $\mathcal{R}^-(y) \succ^{\textit{AF}}_{\mathrm{\Gamma}} \mathcal{R}^-(x)$ then $x \succ^{\textit{AF}}_{\mathrm{\Gamma}} y$.
\end{axiom}

The following two axioms provide two choices: giving precedence to cardinality over quality (i.e. two weakened attackers are worse for the target than one strong attacker), or vice versa. In some situations, both options are rational.
\begin{axiom}[Cardinality Precedence {\sf (CP)}]
A ranked-based semantics $\mathrm{\Gamma}$ satisfies \emph{cardinality precedence} iff for arbitrary argumentation framework $\textit{AF}=\left< \mathcal{X}, \mathcal{R}\right>$, $\forall x,y\in\mathcal{X}$, if $|\mathcal{R}^-(x)| < |\mathcal{R}^-(y)|$ then $x \succ^{\textit{AF}}_{\mathrm{\Gamma}} y$.
\end{axiom}
\begin{axiom}[Quality Precedence {\sf (QP)}]
A ranked-based semantics $\mathrm{\Gamma}$ satisfies \emph{quality precedence} iff for arbitrary argumentation framework $\textit{AF}=\left< \mathcal{X}, \mathcal{R}\right>$, $\forall x,y\in\mathcal{X}$, if $\exists y' \in \mathcal{R}^-(y)$ such that $\forall x'\in\mathcal{R}^-(x)$, $y' \succ^{\textit{AF}}_{\mathrm{\Gamma}} x'$, then $x \succ^{\textit{AF}}_{\mathrm{\Gamma}} y$.
\end{axiom}

The last axiom concerns the way arguments are defended. The consideration is that an argument which is defended against more attackers is ranked higher than an argument which is depended against less attacks. There are two concepts of defense: simple and distributed. The defense of an argument $x$ is \emph{simple} iff every defender of $x$ attacks exactly one attacker of $x$, formally, $\forall y \in \mathcal{D}(x)$ such that $|\mathcal{R}^+(y)\cap\mathcal{R}^-(x)|=1$. The defense of an argument $x$ is \emph{distributed} iff $\forall y \in \mathcal{R}^-(x)$ such that $|\mathcal{R}^-(y)|\geq 1$, i.e., every attacker of $x$ is attacked by at least one argument.
\begin{axiom}[Distributed-Defense Precedence {\sf (DDP)}]
A ranked-based semantics $\mathrm{\Gamma}$ satisfies \emph{distributed-defense precedence} iff for any $\textit{AF}=\left< \mathcal{X}, \mathcal{R}\right>$, $\forall x,y\in\mathcal{X}$ such that $|\mathcal{R}^-(x)|=|\mathcal{R}^-(y)|$ and $|\mathcal{D}(x)|=|\mathcal{D}(y)|$, if the defense of $x$ is simple and distributed and the defense of $y$ is simple but not distributed then $x \succ^{\textit{AF}}_{\mathrm{\Gamma}} y$.
\end{axiom}

In addition, \cite{ref-amgoud2013ranking} provides some relationships between these axioms: if a ranking-based semantics $\mathrm{\Gamma}$ satisfies {\sf (SCT)} then it satisfies {\sf (VP)}; if $\mathrm{\Gamma}$ satisfies both {\sf (CT)} and {\sf (SCT)}, then it satisfies {\sf (DP)}; $\mathrm{\Gamma}$ can not satisfy both {\sf (CP)} and {\sf (QP)}. Now, let us give the following proposition about which axioms the ArgRank satisfies or dose not.
\begin{proposition}
The ranking $\succeq_\alpha$ derived from the counting semantics satisfies {\sf (Ab)}, {\sf (CT)}, {\sf (SCT)}, {\sf (VP)} and {\sf (DP)}, and does not satisfy {\sf (In)}, {\sf (CP)}, {\sf (QP)} and {\sf (DDP)}.
\end{proposition}

The {\sf (Ab)} can be proved by the definition of the counting semantics. By Eqn.~\eqref{Eqn_Iteration}, it follows {\sf (CT)}. In particular, when $\mathcal{R}^-(x_i)\subset\mathcal{R}^-(x_j)$ then the semantics gives $x_i\succ x_j$, which is the case of {\sf (SCT)}. The {\sf (VP)} and {\sf (DP)} can be followed from their relationships with {\sf (CT)} and {\sf (SCT)}. The ranking $\succeq_\alpha$ does not always satisfy {\sf (In)} since an argumentation framework and its weak connected components may have different normalization factor, which may bring about different rankings on arguments.

The ranking $\succeq_\alpha$ does not satisfy {\sf (CP)} since it mainly concerns on the number of the $1$-length attackers. Clearly, the ranking $\succeq_\alpha$ also does not satisfy {\sf (QP)}.

Moreover, {\sf (DDP)} is not satisfied by the ranking $\succeq_\alpha$. As a counter example, for instance, consider the argumentation framework in Fig.~\ref{Fig_CounterExample}, in which the two arguments $x_1$ and $y_1$ have the same number of attackers and defenders: $\mathcal{R}^-(x_1)=\{x_2,x_3\}$ and $\mathcal{R}^-(y_1)=\{y_2,y_3\}$, and $\mathcal{D}(x_2)=\{x_4,x_5\}$ and $\mathcal{D}(y_2)=\{y_4,y_5\}$. We can see that the defense of argument $x_1$ is simple and distributed while the defense of argument $y_1$ is simple but not distributed. Therefore, the axiom {\sf (DDP)} guarantees that $x_1$ is ranked higher than $y_1$. On the other hand, the ranking $\succeq_\alpha$ provides that $\bm{v}_\alpha(x_1)=\bm{v}_\alpha(y_1)=1-\alpha+\frac{\alpha^2}{2}$. This contradicts with the result of {\sf (DDP)}. Assume there is another argument $x_6$ which attacks argument $x_4$, then $x_4$ is weakened, thereby increasing (the rank value of) $x_2$ and decreasing $x_1$, and thereby following $y_1 \succ x_1$. Note that then the defense of $x_1$ still is simple and distributed, therefore, this also confirms that the ranking $\succeq_\alpha$ does satisfy {\sf (DDP)}. The major reason of this counter situation is that the axiom {\sf (DDP)} concentrates too much on quite local topological respects of an argumentation framework, but ignores the global topology. However, the counting semantics is a global approach since the rank value of an argument relies on the rank values of its attackers, which is a recursive definition.

\begin{figure}[htb]
\centering
\subfloat{
\begin{tikzpicture}[->,>=stealth',shorten >=1pt,auto,node distance=1.1cm, semithick]

\tikzstyle{vBlue}=[draw=blue!60,fill=blue!0,circle,text width=5.5mm,inner sep=1pt,minimum height=6pt, align=center]
\tikzstyle{vRed}=[draw=blue!50,fill=black!5,circle,text width=5.5mm,inner sep=1pt,minimum height=6pt, align=center]
\tikzstyle{every edge}=[draw=black!60]

\node[vRed](x1){$x_1$};
\node[vBlue, above left of=x1, yshift=-2mm, xshift=-1.5mm](x2){$x_2$};
\node[vBlue, below left of=x1, yshift=2mm, xshift=-1.5mm](x3){$x_3$};
\node[vBlue, left of=x2](x4){$x_4$};
\node[vBlue, left of=x3](x5){$x_5$};

\path 
      (x2) edge    (x1)
      (x3) edge    (x1)
      (x4) edge    (x2)
      (x5) edge    (x3);
\end{tikzpicture}
\label{Fig_CounterExample_First}}
\hfil
\subfloat{
\begin{tikzpicture}[->,>=stealth',shorten >=1pt,auto,node distance=1.1cm, semithick]

\tikzstyle{vBlue}=[draw=blue!60,fill=blue!0,circle,text width=5.5mm,inner sep=1pt,minimum height=6pt, align=center]
\tikzstyle{vRed}=[draw=blue!50,fill=black!5,circle,text width=5.5mm,inner sep=1pt,minimum height=6pt, align=center]
\tikzstyle{every edge}=[draw=black!60]

\node[vRed](y1){$y_1$};
\node[vBlue, above left of=y1, yshift=-2mm, xshift=-1.5mm](y2){$y_2$};
\node[vBlue, below left of=y1, yshift=2mm, xshift=-1.5mm](y3){$y_3$};
\node[vBlue, left of=y2](y4){$y_4$};
\node[vBlue, left of=y3](y5){$y_5$};

\path (y2) edge    (y1)
      (y3) edge    (y1)
      (y4) edge    (y2)
      (y5) edge    (y2);
\end{tikzpicture}
\label{Fig_CounterExample_Second}}
\caption{A counter-example of axiom {\sf (DDP)}}
\label{Fig_CounterExample}
\end{figure}
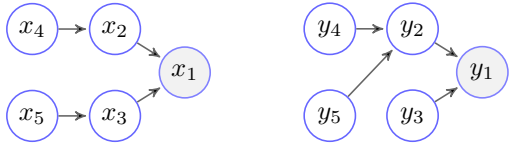
In \cite{ref-pfa2014caterank}, the authors have shown that the categoriser-based ranking semantics satisfies {\sf (Ab)}, {\sf (In)}, {\sf (VP)}, {\sf (DP)}, {\sf (CT)} and {\sf (SCT)}, and does not satisfy {\sf (CP)}, {\sf (QP)} and {\sf (DDP)}. Now, we can find a difference between the categoriser function and the counting semantics, i.e., the ranking $\succeq_\alpha$ does not satisfy {\sf (In)}. This is because the counting semantics adopts a dynamic normalization factor.

\section{Conclusion} \label{Sec_Conclusion}
In this paper, we carry on our previous work on a novel semantical notion for abstract argumentation frameworks: counting semantics. Instead of extensions, the counting semantics assigns strength values (representing degrees of acceptability) to arguments by counting the number of attackers and defenders for arguments. We discuss how to select the damping factor for the counting semantics. We show several interesting relationships between the counting semantics and Dung's classical semantics by means of boolean algebra. We relate the counting semantics with two proposals of generic gradual valuations. Finally, we present an axiomatic analysis of the ranking $\succeq_\alpha$ derived from our counting semantics.


The counting semantics generalizes Dung's extension-based semantics for abstract argumentation and allows for a more fine-grained differentiation of the status of arguments. It is an objective evaluation model (merely based on the graph structure of the argumentation framework). For future work, we intend to extend our semantics to some sophisticated problems, such as the large-scale online debates \cite{ref-leite2011SocialAF}, in which thousands of people participate in and thousands of arguments are expressed, and the bipolar argumentation frameworks \cite{ref-gradual-bipolar}, where two kinds of interaction between arguments: a positive interaction (an argument support another argument) and a negative interaction (an argument can attack another argument) are considered.

\section*{Acknowledgments}
The research reported here was supported by the National Natural Science Foundation of China under contract number NSFC61171121 and NSFC61572279.

\bibliographystyle{IEEEtran}
\bibliography{IEEEabrv,sigproc}

\begin{thebibliography}{10}
\providecommand{\url}[1]{#1}
\csname url@samestyle\endcsname
\providecommand{\newblock}{\relax}
\providecommand{\bibinfo}[2]{#2}
\providecommand{\BIBentrySTDinterwordspacing}{\spaceskip=0pt\relax}
\providecommand{\BIBentryALTinterwordstretchfactor}{4}
\providecommand{\BIBentryALTinterwordspacing}{\spaceskip=\fontdimen2\font plus
\BIBentryALTinterwordstretchfactor\fontdimen3\font minus
  \fontdimen4\font\relax}
\providecommand{\BIBforeignlanguage}[2]{{%
\expandafter\ifx\csname l@#1\endcsname\relax
\typeout{** WARNING: IEEEtran.bst: No hyphenation pattern has been}%
\typeout{** loaded for the language `#1'. Using the pattern for}%
\typeout{** the default language instead.}%
\else
\language=\csname l@#1\endcsname
\fi
#2}}
\providecommand{\BIBdecl}{\relax}
\BIBdecl

\bibitem{ref-pu2015counting}
F.~Pu, J.~Luo, Y.~Zhang, and G.~Luo, ``Attacker and defender counting approach
  for abstract argumentation,'' in \emph{Proceedings of the 37th Annual Meeting
  of the Cognitive Science Society. Austin, TX: Cognitive Science Society},
  2015, pp. 1913--1918.

\bibitem{ref-David00argknowledge}
D.~C. David, D.~Robertson, and J.~Lee, ``Argument-based applications to
  knowledge engineering,'' \emph{The Knowledge Engineering Review}, vol.~15, p.
  2000, 2000.

\bibitem{ref-Dung1995AAF}
P.~M. Dung, ``On the acceptability of arguments and its fundamental role in
  nonmonotonic reasoning, logic programming and n-person games,'' \emph{Journal
  of Artificial Intelligence}, vol.~77, no.~2, pp. 321--357, Sep. 1995.

\bibitem{ref-caminada2006issue}
M.~Caminada, ``On the issue of reinstatement in argumentation,'' in
  \emph{Logics in artificial intelligence}.\hskip 1em plus 0.5em minus
  0.4em\relax Springer, 2006, pp. 111--123.

\bibitem{Simari2009argame}
S.~Modgil and M.~Caminada, ``\BIBforeignlanguage{English}{Proof theories and
  algorithms for abstract argumentation frameworks},'' in
  \emph{\BIBforeignlanguage{English}{Argumentation in Artificial
  Intelligence}}, G.~Simari and I.~Rahwan, Eds.\hskip 1em plus 0.5em minus
  0.4em\relax Springer US, 2009, pp. 105--129.

\bibitem{ref-caminada2009argame}
M.~Caminada and Y.~Wu, ``An argument game for stable semantics,'' \emph{Logic
  Journal of IGPL}, p. jzn029, 2009.

\bibitem{ref-west2001introduction}
D.~B. West \emph{et~al.}, \emph{Introduction to graph theory}.\hskip 1em plus
  0.5em minus 0.4em\relax Prentice hall Englewood Cliffs, 2001, vol.~2.

\bibitem{ref-rienstra2013opponent}
T.~Rienstra, M.~Thimm, and N.~Oren, ``Opponent models with uncertainty for
  strategic argumentation,'' in \emph{Proceedings of the Twenty-Third
  international joint conference on Artificial Intelligence}.\hskip 1em plus
  0.5em minus 0.4em\relax AAAI Press, 2013, pp. 332--338.

\bibitem{ref-horn2012matrix}
R.~A. Horn and C.~R. Johnson, \emph{Matrix analysis}.\hskip 1em plus 0.5em
  minus 0.4em\relax Cambridge university press, 2012.

\bibitem{ref-AC98}
L.~Amgoud and C.~Cayrol, ``On the acceptability of arguments in
  preference-based argumentation,'' in \emph{Proceedings of the Fourteenth
  conference on Uncertainty in artificial intelligence}.\hskip 1em plus 0.5em
  minus 0.4em\relax Morgan Kaufmann Publishers Inc., 1998, pp. 1--7.

\bibitem{ref-cayrol2005graduality}
C.~Cayrol and M.-C. Lagasquie-Schiex, ``Graduality in argumentation,'' \emph{J.
  Artif. Intell. Res.(JAIR)}, vol.~23, pp. 245--297, 2005.

\bibitem{ref-jakobovits1999robust}
H.~Jakobovits and D.~Vermeir, ``Robust semantics for argumentation
  frameworks,'' \emph{Journal of logic and computation}, vol.~9, no.~2, pp.
  215--261, 1999.

\bibitem{ref-besnard2001logic}
P.~Besnard and A.~Hunter, ``A logic-based theory of deductive arguments,''
  \emph{Artificial Intelligence}, vol. 128, no.~1, pp. 203--235, 2001.

\bibitem{ref-amgoud2013ranking}
L.~Amgoud and J.~Ben-Naim, ``Ranking-based semantics for argumentation
  frameworks,'' in \emph{Scalable Uncertainty Management}.\hskip 1em plus 0.5em
  minus 0.4em\relax Springer, 2013, pp. 134--147.

\bibitem{ref-pfa2014caterank}
F.~Pu, J.~Luo, Y.~Zhang, and G.~Luo, ``Argument ranking with categoriser
  function,'' in \emph{Knowledge Science, Engineering and Management}.\hskip
  1em plus 0.5em minus 0.4em\relax Springer, 2014, pp. 290--301.

\bibitem{ref-leite2011SocialAF}
J.~Leite and J.~Martins, ``Social abstract argumentation,'' in
  \emph{Proceedings of the Twenty-Second international joint conference on
  Artificial Intelligence-Volume Volume Three}.\hskip 1em plus 0.5em minus
  0.4em\relax AAAI Press, 2011, pp. 2287--2292.

\bibitem{ref-gradual-bipolar}
C.~Cayrol and M.-C. Lagasquie-Schiex, ``Gradual valuation for bipolar
  argumentation frameworks,'' in \emph{Symbolic and Quantitative Approaches to
  Reasoning with Uncertainty}.\hskip 1em plus 0.5em minus 0.4em\relax Springer,
  2005, pp. 366--377.

\end{thebibliography}

\end{document}